\documentclass[review]{elsarticle}

\usepackage{lineno,hyperref}
\newtheorem{example}{\bf{Example}}
\newtheorem{theorem}{\bf{Theorem}}
\newtheorem{proof}{Proof}[section]
\hypersetup{hidelinks}
\modulolinenumbers[5]

\journal{arXiv.org}

\begin{document}

\begin{frontmatter}

\title{On the negation of a Dempster-Shafer belief structure based on maximum uncertainty allocation}


\author[mymainaddress]{Xinyang Deng\corref{mycorrespondingauthor}}
\ead{xinyang.deng@nwpu.edu.cn}
\cortext[mycorrespondingauthor]{Corresponding author}

\author[mymainaddress]{Wen Jiang\corref{mycorrespondingauthor}}
\ead{jiangwen@nwpu.edu.cn}

\address[mymainaddress]{School of Electronics and Information, Northwestern
Polytechnical University, Xi'an, 710072, China}

\begin{abstract}
Probability theory and Dempster-Shafer theory are two germane theories to represent and handle uncertain information. Recent study suggested a transformation to obtain the negation of a probability distribution based on the maximum entropy. Correspondingly, determining the negation of a belief structure, however, is still an open issue in Dempster-Shafer theory, which is very important in theoretical research and practical applications. In this paper, a negation transformation for belief structures is proposed based on maximum uncertainty allocation, and several important properties satisfied by the transformation have been studied. The proposed negation transformation is more general and could totally compatible with existing transformation for probability distributions.
\end{abstract}


\end{frontmatter}


\section{Introduction}

Uncertainty modelling is an important aspect in knowledge representation. A variety of theories have been developed to express the uncertainty contained in pieces of information. Dempster-Shafer theory \cite{Dempster1967,Shafer1976}, also called belief function theory, is a widely used mathematical tool to express the uncertainty of ambiguity \cite{Chen2013On,jousselme2006measuring,jiang2018Correlation,xiao2018multi}. In many knowledge-based or expert-based systems, Dempster-Shafer theory is usually employed to represent uncertain information or human being's knowledge, and assist in reasoning and decision-making.

In mathematical logic, ``AND" ($\wedge$), ``OR" ($\vee$), and ``NOT" ($\neg$) are three basic operations. In a recent work \cite{yager2015on}, Yager has studied the negation of a probability distribution from the perspective of maximum entropy, in order to represent the knowledge contained in the negation of a probability distribution. In theory and practical applications, probability theory is very useful for representing randomness. Compared to probability theory, Dempster-Shafer theory is able to express not only randomness but also nonspecificity, which is implemented by belief structures, the basic form of representing information or knowledge in Dempster-Shafer theory. The issue of determining the negation of a belief structure, however, still remains to be unsolved so far. In this paper, we are concerned with the negation of a belief structure to fill up this research gap.

This study is theoretically motivated by two aspects. At first, with the extensive use of Dempster-Shafer theory, it requires to reveal basic concepts in this theory, including the negation of a belief structure. Secondly, research to the negation of belief structures is beneficial for further enhancing the reasoning ability of Dempster-Shafer theory since there is a logical equivalence between entailment and negation, namely $p \to q \Leftrightarrow \neg p \vee q$, while the concept of entailment has been attracted much attention in previous studies \cite{entailment1986,yager2012entailment,Entailment2018}.

Motivated by the reasons mentioned above, in this paper a negation transformation for belief structures is proposed based on maximum uncertainty allocation. The suggested negation transformation has two desirable properties. First, it leads to an increase of uncertainty degree contained in the information of negation until an attractor has been reached in most cases. Second, the proposed negation transformation of belief structures is theoretically consistent with Yager's negation of probability distributions, and can be reduced to Yager's negation in a special case.

\section{Basics of Dempster-Shafer theory}
Assume there is a random variable $X$ taking values from $\Theta = \{\theta_1, \theta_2, \cdots, \theta_n\}$. In Dempster-Shafer theory, set $\Omega$ is called a frame of discernment (FOD), and a belief structure is a mapping $m$ from the power set of $\Theta$, denoted as ${2^\Theta }$, to interval $[0, 1]$, satisfying
\begin{equation}
m(\emptyset ) = 0 \quad {\rm{and}} \quad \sum\limits_{A \in 2^\Theta }{m(A) = 1}.
\end{equation}
If $m(A) > 0$, then $A$ is called a focal element, and $m(A)$ measures the belief assigned exactly to $A$ and represents how strongly the evidence supports $A$. The union of all focal elements is called the core of belief structure $m$.

Associated with a belief structure $m$, belief measure $Bel$ and plausibility measure $Pl$  express the lower bound and upper bound of the support degree of each set $A$, $A \subseteq \Theta$, respectively. They are defined as
\begin{equation}
Bel(A) = \sum\limits_{B \subseteq A} {m(B)},
\end{equation}
\begin{equation}
Pl(A) = 1 - Bel(\bar A) = \sum\limits_{B \cap A \ne \emptyset }{m(B)},
\end{equation}
where $\bar A = \Theta  - A$. Obviously, $Pl(A) \ge Bel(A)$ for each $A \subseteq \Theta$. The belief $Bel(A)$ and plausibility $Pl(A)$ constitute a belief interval $[Bel(A), Pl(A)]$. The length of $[Bel(A), Pl(A)]$ represents the degree of imprecision for the proposition or focal element $A$.

\section{Proposed negation of a Dempster-Shafer belief structure}
The negation is a basic operation in logic. Probability distribution is a way to describe the state of an object, so does Dempster-Shafer belief structure. Yager \cite{yager2015on} has given the negation transformation for a probability distribution, the negation of a belief structure, however, is an open issue.

In this paper, the idea of obtaining the negation of a Dempster-Shafer belief structure is initially from the investigation of the negation of a set.

For a random variable $X$ whose possible states make up a set $\Theta = \{\theta_1, \theta_2, \cdots, \theta_n\}$,
\begin{itemize}
  \item If it is observed that the state of $X$ is a singleton $\theta_i$, denoted as $s(X) = {\theta_i}$, then the negation of this observation is that the state of $X$ is $\Theta - \theta_i$, namely $\bar s(X) = \{\Theta - \theta_i\}$.
  \item If it observes that the state of $X$ is either $\theta_i$ or $\theta_j$, which is uncertain, denoted as $s(X) = \{\theta_i, \theta_j\}$, then there are two cases. First, if the true state of $X$ is $\theta_i$, then $\bar s(X) = \{\Theta - \theta_i\}$; Second, if $X$'s true state is $\theta_j$, then $\bar s(X) = \{\Theta - \theta_j\}$. Therefore, the negation of the observation $s(X) = \{\theta_i, \theta_j\}$ is either $\{\Theta - \theta_i\}$ or $\{\Theta - \theta_j\}$, thus $\bar s(X) = \{\Theta - \theta_i\} \cup \{\Theta - \theta_j\}$.
  \item Further, if the observation to the state of $X$ is indicated by $s(X) = A$, where $A \subseteq \Theta$, according to the above analysis, the negation of $s(X) = A$ is obtained as $\bar s(X) = \bigcup\limits_{\forall \theta \in A } {(\Theta  - \theta)} $.
\end{itemize}

Based on this understanding, given a belief structure $m$ defined over $\Theta$, for its each component $m(A_i) = \alpha_i$, the negation, denoted as $\bar m( \bar {A_i})$, is defined by the following
\begin{itemize}
  \item If $A_i$ is a singleton $\theta$, $\bar m( \bar {A_i}) = \alpha_i$ where $\bar {A_i} = \Theta - A$.
  \item If $A_i$ is not a singleton, $\bar m( \bar {A_i}) = \alpha_i$ where $\bar {A_i} = \bigcup\limits_{\forall \theta \in A_i } {(\Theta  - \theta)} $.
\end{itemize}

Let $\bar m$ be the negation of a belief structure $m$, therefore $\bar m$ is defined as
\begin{equation}\label{EqNegationofBS}
\bar m(B) = \sum\limits_{{A_i}\;{\rm{satisfying}}\; ( \bigcup\limits_{\forall \theta \in {A_i}} {(\Theta  - \theta)} \;\; )  = B } {m({A_i})} \; ,
\end{equation}
where $B \subseteq \Theta$.

The above negation transformation is mainly be understood through the negation of focal elements. Hence, in order to obtain the negation of a belief structure $m$, where $m({A_i}) = {\alpha _i}$ satisfying ${\alpha _i} > 0$ and $\sum\limits_i {{\alpha _i}}  = 1$, according to Eq. (\ref{EqNegationofBS}), the following three steps can be executed sequentially.

At first, obtain the negation of each focal element $A_i$ in belief structure $m$, which is denoted as $\bar {A_i}$;

Second, assign the mass of $A_i$ in $m$ to $\bar {A_i}$ in $\bar m$, namely let $\bar m( \bar {A_i}) = \alpha_i$.

At last, merge all masses for each focal element $\bar {A_i}$, ${A_i} \subseteq \Theta$.

\begin{example}
A belief structure $m$ is defined over $\Theta = \{a,b,c\}$:

$m({A_1}) = 0.7,\quad m({A_2}) = 0.1,\quad m({A_3}) = 0.2,$

where

${A_1} = \{ a\} ,\quad {A_2} = \{ b,c\} ,\quad {A_3} = \{ a,b,c\}. $

According to the proposed transformation, the negation of $m$ is obtained by:

$\bar m(\bar {A_1}) = 0.7,\quad \bar m(\bar {A_2}) = 0.1,\quad \bar m(\bar {A_3}) = 0.2,$

with

$\bar {A_1} = \{ b,c\} ,\quad \bar {A_2} = \{a, b,c\} ,\quad \bar {A_3} = \{ a,b,c\}. $

Namely $\bar m(\{ b,c\}) = 0.7,\quad \bar m(\{ a,b,c\}) = 0.3.$

\end{example}

Next, let us use a simpler perspective on Eq. (\ref{EqNegationofBS}). Actually, for a focal element $A$, its negation is $\bar A = \Theta - \theta$ if $A$ is a singleton indicated by $\theta$, and $\bar {A} = \bigcup\limits_{\forall \theta \in A } {(\Theta  - \theta)} = \Theta$ if $|A| \ge 2$. Namely, the negation of $A$ is either $\Theta - \theta$ or $\Theta$. Therefore, given a belief structure $m$, its negation $\bar m$ can be simply represented by
\begin{equation}\label{EqNegationofBSnewform}
\bar m(B) = \left\{ \begin{array}{l}
 m(\theta ),\qquad \quad \qquad B = \Theta  - \theta, \; \forall \theta  \\
 \sum\limits_{\forall A,|A| \ne 1} {m(A)} ,\quad B = \Theta \\
 \end{array} \right.
\end{equation}

\section{Properties of the negation transformation}
In this section, the properties of the proposed negation transformation of belief structures is analyzed. Mainly, the change of uncertainty degree from original belief structure $m$ to its negation $\bar m$ is discussed.

In Dempster-Shafer theory, measuring the uncertainty degree pertaining to a belief structure is an open issue. Various measures have been proposed in previous studies, for example aggregated uncertainty (AU) \cite{harmanec1994measuring224}, ambiguity measure (AM) \cite{jousselme2006measuring}, new definition of entropy of belief functions in \cite{jirouvsek2018new}, distance-based total uncertainty measure ${TU}^{I}$ \cite{yang2015newKBS}, ${TU}^{I}_E$ as an improved ${TU}^{I}$ \cite{XDIJIS21999}, to name but a few \cite{deng2016deng91,abellan2008requirements,wang2017uncertaintySY}. In this paper, we do not center on concrete formulas of uncertainty measures of belief structures, but start from basic features of a rational uncertainty measure to study this issue.

As been widely accepted, Dempster-Shafer theory can model more diverse types of uncertainty, including randomness and non-specificity \cite{yager1983entropy94249}, than probability theory which is mainly used to represent randomness. In probability theory, the equal-probability distribution has the maximum uncertainty degree. However, in Dempster-Shafer theory, it is not appropriate to define that the Bayesian belief structure with equal basic probabilities possesses the maximum uncertainty degree. As been more recognized, given a FOD $\Theta$ the vacuous belief structure $m_{\Theta}$, namely $m_{\Theta} (\Theta) = 1$, has the maximum uncertainty degree in Dempster-Shafer theory. Here, by following this standpoint, we assume an uncertainty degree measure of belief structures, denoted as $H$, which obtains the maximum value on $m_{\Theta}$.


Based on the uncertainty measure $H$ assumed above, the following theorems are satisfied by the proposed negation transformation of belief structures.

\begin{theorem}\label{Theorem1}
Given a belief structure $m$ with $m({A_i}) = {\alpha _i}$ satisfying ${\alpha _i} > 0$ and $\sum\limits_i {{\alpha _i}}  = 1$, defined over $\Theta$ where $|\Theta| > 2$, let $\bar m$ represent the negation of $m$, then $H(\bar m) \ge H(m)$.
\end{theorem}

\begin{proof}
Since $H(m)$ gets the maximum value while $m = m_{\Theta}$, we can have

\[H(m(\Theta ) = \alpha ) > H(m({A_i}) = \alpha ) \; {\rm{if}} \; {A_i} \neq \Theta,\]
where $H(m(A ) = \alpha )$ represents the uncertainty caused by component $m(A ) = \alpha$, $A \subseteq \Theta$, and the uncertainty pertaining to $m(\Theta  - \theta) = \alpha$ is larger than that of $m(\theta) = \alpha$, because $(\Theta  - \theta)$ owns higher non-specificity compared to $\theta$ which has no non-specificity while $|\Theta| > 2$.

After the negation transformation, for $m$ and $\bar m$, here we have
\[H\left( {\bar m(\Theta ) = \sum\limits_{\forall {A_i},|{A_i}| \ne 1} {m({A_i})} } \right) \ge \sum\limits_{\forall {A_i},|{A_i}| \ne 1} {H\left( {m({A_i}) = {\alpha _i}} \right)} ,\]
and
\[H\left( {\bar m(B) = m({A_i})} \right) > H\left( {m({A_i}) = {\alpha _i}} \right)\]
where $B = \Theta  - A_i$ and $|A_i| = 1$.

Therefore, we have $H(\bar m) \ge H(m)$.
\end{proof}

While $|\Theta| = 2$, Theorem \ref{Theorem1} degenerates to the following form.
\begin{theorem}\label{Theorem2}
Assume $m$ is a belief structure defined on $\Theta$ with $|\Theta| = 2$, then $H(\bar m) = H(m)$.
\end{theorem}

The proof of Theorem \ref{Theorem2} is omitted since it can be easily verified.

From Theorem \ref{Theorem1} and \ref{Theorem2}, we can find that the uncertainty degree of the negation of a belief structure is nondecreasing, compared to the original belief structure. The proposed negation transformation always moves to the direction of uncertainty increasing of belief structures. The following two theorems give the end point of this negation transformation.

\begin{theorem}\label{Theorem3}
Assume $\bar m$ is the negation of belief structure $m$ defined on $\Theta$ with $|\Theta| > 2$ and $\bar {\bar m}$ is the negation of $\bar m$, then $\bar {\bar m} = m_{\Theta}$.
\end{theorem}

\begin{theorem}\label{Theorem4}
Assume $m$ is a belief structure defined on $\Theta$ with $|\Theta| = 2$, then $\bar {\bar m} = m$.
\end{theorem}

In terms of Eq. (\ref{EqNegationofBSnewform}), Theorem \ref{Theorem3} and \ref{Theorem4} can be easily proved. since ${\bar m_{\Theta}} = m_{\Theta}$, thus the vacuous belief structure is an attractor of the negation transformation of a belief function in most cases. Just in a special case when $|\Theta| = 2$, $\bar {\bar m} = m$.

\section{Relationship with Yager's negation transformation of probability distributions}
For a probability distribution $P = \{p_1, \cdots, p_n\}$ on set $\Theta = \{\theta_1, \cdots, \theta_n\}$, Yager \cite{yager2015on} have defined its negation as follows:
\begin{equation}
{\bar p_i} = \frac{{1 - {p_i}}}{{n - 1}}.
\end{equation}

According to Yager's negation transformation, the uncertainty of obtained probability distribution is increasing until an equal-probability distribution is obtained. Namely, the attractor of Yager's negation transformation is a maximal entropy allocation of the probabilities. Yager's transformation is rational for a probability distribution.

In Dempster-Shafer theory, a probability distribution $P = \{p_1, \cdots, p_n\}$ can be seen as a Bayesian belief structure denoted as $m_B$. Therefore, according to proposed negation transformation of belief structures, the negation of a Bayesian belief structure $m_B$ is obtained as follows
\[{{\bar m}_B}(\Theta  - {\theta _i}) = m_B(\theta_i) = {p_i},\]
for any $\theta_i \in \Theta$.

Further, by considering that the basic probabilities can not be assigned to sets but singletons in the context of probability theory, the proposed negation transformation of belief structures can be totally reduced to Yager's negation transformation of probability distributions. Based on the consistent idea of maximum uncertainty, the mass of focal element $\Theta  - {\theta _i}$ in $\bar m_B$ will be further assigned to its each elements averagely. So, we can have
\[\begin{array}{l}
 {{\bar m}_B}({\theta _i}) = \sum\limits_{j \ne i} {\frac{{{{\bar m}_B}(\Theta  - {\theta _j})}}{{n - 1}}}  = \frac{{1 - {{\bar m}_B}(\Theta  - {\theta _i})}}{{n - 1}} \\
 \quad \; \quad \quad  = \frac{{1 - {p_i}}}{{n - 1}} = {{\bar p}_i}. \\
 \end{array}\]

Thus, the proposed negation transformation of belief structures can be totally reduced to Yager's negation transformation of probability distributions, if been required to generate a Bayesian-type negation for a Bayesian belief structure.

\section{Conclusion}
In this paper, the negation of a belief structure in Dempster-Shafer theory was investigated, and a negation transformation has been proposed for this issue. The proposed transformation satisfies desirable properties in the aspect of uncertainty change and could be totally reduced to Yager's negation transformation developed for probability distributions. In the future research, the logical equivalence between $p \to q$ and $\neg p \vee q$ will be studied through the proposed negation transformation in the context of Dempster-Shafer theory.

\section*{Acknowledgment}
The work is partially supported by National Natural Science Foundation of China (Program Nos. 61703338, 61671384), Natural Science Basic Research Plan in Shaanxi Province of China (Program No. 2018JQ6085).

\section*{References}

\bibliography{References}

\end{document}